\documentclass[sigconf]{acmart} 

\usepackage{balance} 

\usepackage{algorithm}
\usepackage{ctable}
\usepackage{algorithmic}
\usepackage{mathtools}
\usepackage{thmtools, thm-restate}
\usepackage{amsmath,amsfonts,amsthm,epsfig,epstopdf,array}
\usepackage[inkscapelatex=false]{svg}

\theoremstyle{definition}

\newtheorem{remark}{Remark}
\theoremstyle{plain}

\newtheorem{theorem}{Theorem}

\newtheorem{assumption}{Assumption}

\DeclareMathOperator{\cpt}{\mathbb{C}\mathbb{P}\mathbb{T}}
\DeclareMathOperator{\prob}{\mathbb{P}}
\DeclareMathOperator{\expectation}{\mathbb{E}}

\setcopyright{none}
\acmConference[preprint]{Preprint}{Risk-Sensitive Multi-Agent Reinforcement Learning in Network Aggregative Markov Games}{}{}
\copyrightyear{}
\acmYear{}
\acmDOI{}
\acmPrice{}
\acmISBN{}

\acmSubmissionID{1100}

\title[MARSRL]{Risk-Sensitive Multi-Agent Reinforcement Learning in Network Aggregative Markov Games}

\author{Hafez Ghaemi}
\affiliation{
  \institution{University of Tehran, School of ECE}
  \country{}
  }
\email{hafez.ghaemi@ut.ac.ir}

\author{Hamed Kebriaei}
\affiliation{
  \institution{University of Tehran, School of ECE}
  \country{}
  }
\email{kebriaei@ut.ac.ir}

\author{Alireza Ramezani Moghaddam}
\affiliation{
  \institution{University of Tehran, School of ECE}
  \country{}
  }
\email{a.ramezany@ut.ac.ir}

\author{Majid Nili Ahamdabadi}
\affiliation{
  \institution{University of Tehran, School of ECE}
  \country{}
  }
\email{mnili@ut.ac.ir}

\begin{abstract}
Classical multi-agent reinforcement learning (MARL) assumes risk neutrality and complete objectivity for agents. However, in settings where agents need to consider or model human economic or social preferences, a notion of risk must be incorporated into the RL optimization problem. This will be of greater importance in MARL where other human or non-human agents are involved, possibly with their own risk-sensitive policies. In this work, we consider risk-sensitive and non-cooperative MARL with cumulative prospect theory (CPT), a non-convex risk measure and a generalization of coherent measures of risk. CPT is capable of explaining loss aversion in humans and their tendency to overestimate/underestimate small/large probabilities. We propose a distributed sampling-based actor-critic (AC) algorithm with CPT risk for network aggregative Markov games (NAMGs), which we call Distributed Nested CPT-AC. Under a set of assumptions, we prove the convergence of the algorithm to a subjective notion of Markov perfect Nash equilibrium in NAMGs. The experimental results show that subjective CPT policies obtained by our algorithm can be different from the risk-neutral ones, and agents with a higher loss aversion are more inclined to socially isolate themselves in an NAMG.\footnote{Code available at \url{https://github.com/hafezgh/risk-sensitive-marl-namg}}
\end{abstract}

\keywords{Multi-Agent Reinforcement Learning, Actor-Critic, Aggregative Games, Risk-Sensitivity, Cumulative Prospect Theory}

\newcommand{\BibTeX}{\rm B\kern-.05em{\sc i\kern-.025em b}\kern-.08em\TeX}

\begin{document}


\pagestyle{fancy}
\fancyhead{}


\maketitle 


\section{Introduction}

Markov game (MG) is a common framework for studying multi-agent systems (MAS), and it is the main theoretical framework for multi-agent reinforcement learning (MARL) \cite{shapley1953stochastic,littman1994markov}. In classical MARL, each agent is assumed to have a risk-neutral objective, i.e., it tries to maximize a notion of expected return without taking into account subjective preferences of itself or of the other agents in the MAS. Risk-neutral MARL in MGs has seen great advances in recent years, especially in specific types of MGs, such as zero-sum MGs \cite{sayin2021decentralized,zhang2020model,alacaoglu2022natural,perolat2015approximate,qiu2021provably} and Markov potential games \cite{ding2022independent,fox2022independent,maheshwari2022independent,mguni2021learning,leonardos2021global}. However, the risk-neutral RL objective often falls short when representing agents with distinct subjective preferences, such as internal cognitive biases of themselves or of other agents. Thus, to address these preferences, agents integrate a risk measure into their RL objective, ushering into the realm of risk-sensitive reinforcement learning (RSRL). In general, the literature on risk-sensitive MARL is more sparse compared to single-agent RSRL. The majority of the works that consider risk-sensitive multi-agent MDPs are concerned not with an RL setting but either with theoretically proving the existence of Markov perfect Nash equilibria, or finding these equilibria using iterative algorithms given complete information of the game in a centralized setting for MDPs with specific constraints \cite{ghosh2023discrete,zhong2022chance,ghosh2022nonzero,pal2021zero,bacsar2021robust,zhang2021derivative}.

Risk in RL can be categorized into two main types based on the risk-sensitive objective, as delineated by Prashanth and Fu \citep{prashanth2022risk}. The first category, explicit risks, involves directly incorporating the risk measure into the objective function. In contrast, implicit risks are integrated by imposing a constraint on the RL stochastic optimization problem. Notably, in practice, implicit risk-sensitive objectives are often transformed into explicit objectives. This is achieved by formulating a Lagrangian and computing its gradient to employ algorithms founded on policy gradient (PG) methods \citep{prashanth2022risk}. Within the spectrum of implicit risk measures in RL and MDPs, notable examples include variance as risk (\citep{tamar2012policy,tamar2013temporal,prashanth2016variance} in single-agent RSRL, and \cite{reddy2019risk} in risk-sensitive MARL), and chance constraints (\citep{chow2017risk} in single-agent RSRL). On the other hand, explicit risk measures encompass entropic risk measures predicated on exponential return (\citep{borkar2001sensitivity,fei2021exponential,moharrami2022policy} in single-agent RSRL and \cite{noorani2022risk,soorki2021ultra} in risk-sensitive MARL), coherent risk measures, and cumulative prospect theory (CPT).

Coherent risk measures \citep{artzner1999coherent,delbaen2002coherent}, such as the well-known conditional value at risk (CVaR), mean semi-deviation \cite{shapiro2021lectures}, and spectral risk \citep{acerbi2002spectral}, are widely used in the fields of economy and operations research. Their application has also been explored within MDPs as dynamic risk measures. Osogami et al. \cite{osogami2012robustness} showed that risk-sensitive MDPs governed by Markov coherent risk measures can be classified under the domain of robust MDPs. Subsequently, dynamic programming methodologies have been suggested for this type of MDPs \cite{ruszczynski2010risk,cavus2014risk}. Building on these works, PG-based techniques and actor-critic (AC) algorithms have also been developed for RSRL with coherent risk measures, as detailed in \citep{la2013actor,tamar2014scaling,chow2014algorithms,tamar2015policy} for single-agent RSRL, and in \cite{zhu2022nash,munir2021risk,qiu2021rmix} for risk-sensitive MARL.

\paragraph{\textbf{CPT Background.}} The concept of Prospect Theory (PT) emerged as an alternative model to expected utility theory, providing a more accurate model of human decision-making under uncertainty \citep{kahneman1979prospect}. To enhance the applicability of PT, Cumulative Prospect Theory (CPT) was subsequently introduced \citep{tversky1992advances}. Unlike PT, CPT applies weighting functions to cumulative probabilities, addressing them separately for positive and negative outcomes. By integrating these probability weighting functions and a non-linear utility function, CPT successfully illustrates varying human attitudes towards potential gains and losses against a subjective reference point. Central to CPT is the idea that humans typically exhibit aversion to losses, i.e., they generally take more risks when facing potential gains and take fewer risks when confronted with potential losses. Additionally, CPT's framework elucidates human inclinations to overestimate small probabilities and underestimate large ones during uncertain decision-making. When we consider CPT in the context of either static or dynamic Markov risk measures, it meets only two of the four requirements that define a coherent risk measure. A risk measure, when applied to a random variable (r.v.) representing potential outcomes, is deemed coherent if it has the following four characteristics: convexity, monotonicity, translation invariance, and positive homogeneity \cite{artzner1999coherent}. Among these, the CPT risk measure only possesses monotonicity and positive homogeneity, and is neither translation invariant nor convex. The non-coherent nature of CPT makes it more challenging to work with mathematically. CPT can be seen as a generalization of coherent risk measures, i.e., by appropriate selection of CPT probability weighting functions, one can derive various coherent risk measure formulations \cite{lin2013dynamic,jie2018stochastic}.

\paragraph{\textbf{Contributions.}} In this work, we consider risk-sensitive MARL with CPT risk measure in network aggregative Markov games (NAMGs), and propose a distributed actor-critic algorithm to find risk-sensitive policies for each agent. We derive a policy gradient theorem for CPT MARL based on a subjective steady-state distribution of the MDP from each agent's prespective, and provide a sampling-based approach to estimate the value functions with asymptotic consistency. Since CPT is a generalization of coherent risk measures, our PG theorem generalizes the previous PG works for static and dynamic coherent risk measures \cite{chow2014algorithms,tamar2015policy}. Under a set of assumptions, we prove the convergence of our algorithm to a subjective and risk-sensitive notion of Markov perfect Nash equilibrium (MPNE) which we show is unique given the aforementioned assumptions. Experimentally, we also demonstrate that a higher loss aversion can make agents more conservative and increase their tendency for social isolation in an NAMG.

\begin{remark}{(\emph{Application})}
A potential application of the proposed framework is calculating CPT risk-sensitive policies of human agents in real-world settings, such as driving scenarios or financial markets, that can be modeled by NAMGs. Subsequently, these policies can serve dual purposes: guiding agents towards strategies optimized for their individual preferences or facilitating social or economic changes in the environment to steer agents in a direction that aligns with desired outcomes.
\end{remark}

\section{Related Works}
In the context of Markov risk measures in MDPs, CPT is articulated through two distinct formulations. The first formulation is the nested structure, wherein the CPT operator is applied to the cumulative return after each step (action taken) \citep{lin2013dynamic,lin2013stochastic,lin2018probabilistically}. An important advantage of this formulation is that it ensures the existence of a Bellman optimality equation. Recently, Tian et al. \cite{tian2021bounded} extended this nested formulation to a multi-agent setting with agents that are characterized by bounded rationality and operating under quantal level-$k$ strategies \citep{wright2017predicting}. Restricting their approach to deterministic policies, they propose a centralized value iteration algorithm to determine optimal risk-sensitive policies given a complete model of the environment and the reward functions.

In the second formulation, the CPT operator is applied solely to the agent's final cumulative return at the end of every episode \citep{prashanth2016cumulative,jie2018stochastic}. Contrary to the nested formulation, this formulation does not have a Bellman equation. However, it can be approached from a stochastic optimization perspective, allowing policy optimization through a gradient-based method akin to PG techniques \cite{jie2018stochastic}. This PG method has also been implemented by considering neural networks for policy approximation \citep{markowitz2021deep}. It is important to emphasize that the absence of the Bellman equation in this context necessitates policy optimization exclusively through offline Monte Carlo sampling constrained by a finite time horizon.

To date, no cognitive research has been conducted to ascertain which of the two CPT RSRL formulations best represents the dynamic risk behavior exhibited by humans, whether in single-agent or multi-agent environments. Nonetheless, the following can be said about the two formulations:

\begin{itemize}
    \item The nested formulation benefits from the presence of a Bellman equation, enabling the use of online actor-critic algorithms and a recursively defined value function. This advantage is absent in the non-nested formulation, where it is only plausible to use PG techniques using offline Monte Carlo sampling.
    \item In both formulations, due to the substitution of the expectation operator with the non-linear CPT operator, it is possible for the optimal policy to exhibit non-deterministic characteristics \cite{lin2018probabilistically,jie2018stochastic} even in single-agent RL.
    \item The non-nested formulation aligns well with finite-horizon episodic tasks where the agent is rewarded at the end of each episode. However, its applicability is limited when considering infinite-horizon tasks. Conversely, the nested formulation and its Bellman equation are suitable for tasks where the agent is rewarded at every timestep.
    \item In scenarios without complete information of the model, the reward function, or the policies of other agents, both formulations necessitate a strategy for estimating the CPT value given that we have access to a simulator of the MDP or a large enough experience dictionary (replay buffer). Such an estimation technique tailored for the non-nested formulation has been introduced by Jie et al. \cite{jie2018stochastic}.
\end{itemize}

\begin{remark}{(\emph{Motivation})}
Given the above considerations, in risk-sensitive MARL with CPT, in a setting where the agents interact in an online infinite-horizon MDP with limited information about other agents' policies, the nested CPT formulation is the viable option to adopt. Due to the possibility of non-deterministic optimal policies for each agent in MARL, we opt for actor-critic style algorithms using parameterized policies. Furthermore, we consider NAMGs as our MARL framework due to three reasons. First, because they are inherently suited to distributed algorithms. Second, given a set of assumptions, NAMG, and its risk-sensitive version can be shown to have a unique Markov perfect Nash equilibrium which our algorithm converges to. And third, because NAMGs are a suitable framework to show the tangible effect of loss aversion in human-like agents and on their tendency for social isolation and conservatism.
\end{remark}

\section{Preliminaries}
\subsection{Cumulative Prospect Theory}
Given a real-valued r.v. $X$ with distribution $\prob(X)$, a reference point $x_0$, two monotonically non-decreasing weighting functions, $\omega^+: [0,1] \rightarrow [0,1], \omega^-: [0,1] \rightarrow [0,1]$, utility functions $u^+: \mathbb{R}^+ \rightarrow \mathbb{R}^+, u^-: \mathbb{R}^-\rightarrow \mathbb{R}^+$, and given appropriate integrability assumptions, we can define the CPT value using Choquet integrals as

\begin{equation}
\begin{split}
    \cpt_{\prob}[X] &:= \int^{\infty}_0 \omega^+(\prob(u^+((X-x_0)_+)>x))dx-\\ &\int^{\infty}_0 \omega^-(\prob(u^-((X-x_0)_-)>x))dx., 
\end{split}
\end{equation}

where we denote $(.)_+=max(0,.)$ and $(.)_+=-min(0,.)$. For a discrete r.v., we can define the CPT value similarly as 
\begin{subequations} \label{eq:cptdef}
\begin{equation}
\begin{split}
    \cpt_{\prob}[X] &:= \sum^n_{i=0} \phi^+(\prob{P}(X=x_i)) u^+(x_i-x^0) \\
    &- \sum^{-1}_{-m} \phi^-(\prob(X=x_i)) u^-(x_i-x^0),
\end{split}
\end{equation}
\begin{equation}
\begin{split}
    \phi^+(\prob(X=x_i)) &= \omega^+\left(\sum_{j=i}^n \prob(X=x_j)\right)\\&-\omega^+\left(\sum_{j=i+1}^n \prob(X=x_j))\right),
\end{split}
\end{equation}
\begin{equation}
\begin{split}
    \phi^-(\prob(X=x_i)) &=\omega^-\left(\sum_{j=-m}^i \prob(X=x_j)\right)\\&-\omega^-\left(\sum_{j=-m}^{i-1} \prob(X=x_j)\right),
\end{split}
\end{equation}
\end{subequations}

where $x_0$ serves as a reference point that separates gains and losses. Without loss of generality, we assume $x_0=0$ throughout this paper. Conventional representations of CPT weighting functions include $\omega^+(p) = \frac{p^{\gamma}}{(p^{\gamma} + (1-p)^{\gamma})^{(1/\gamma)}}$ and $\omega^-(p) = \frac{p^{\delta}}{(p^{\delta} + (1-p)^{\delta})^{(1/\delta)}}$ \citep{tversky1992advances}, or $\omega^+(p) = \exp(-(-ln p)^{\gamma})$ and $\omega^-(p) = \exp(-(-ln p)^{\delta})$ \citep{prelec1998probability}. Note that by setting $\delta$ and $\gamma$ equal to $1$, the definition of expected utility $E_{\prob}[u(X)]$ is recovered which shows that CPT is a generalization of expected utility theory. Furthermore, $u^+$ and $u^-$ are usually concave functions ($-u^-$ is convex) to reflect the higher sensitivity of humans towards losses compared to gains \cite{kahneman1979prospect}. As a result, the utility function can have analytical representations $u^+(x)=x^{\alpha}$ if $x\geq 0$, and $u^-(x)=\lambda (-x)^{\beta}$ if $x<0$. The parameters $\gamma, \delta, \alpha, \beta,$ and $\lambda$ are subjective model parameters that can differ from person to person based on their level of risk-aversion and individual characteristics. The conventional representations of weighting and utility functions given a set of subjective parameters are plotted in Figures~\ref{fig:wieghting} and \ref{fig:util}.

\begin{figure}[h]
 \centering
 \includegraphics[width=5cm,height=5cm,keepaspectratio]{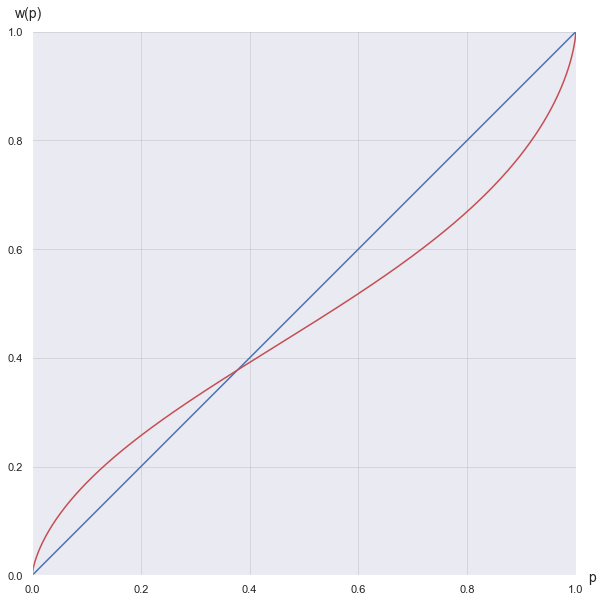}
 \caption{Conventional CPT weighting functions; $\omega^+(p) = \frac{p^{\gamma}}{(p^{\gamma} + (1-p)^{\gamma})^{(1/\gamma)}}$ and $\omega^-(p) = \frac{p^{\delta}}{(p^{\delta} + (1-p)^{\delta})^{(1/\delta)}}$ with $\gamma=\delta=0.69$.}
 \label{fig:wieghting}
\end{figure}
\begin{figure}[h]
 \centering
 \includegraphics[width=5cm,height=5cm,keepaspectratio]{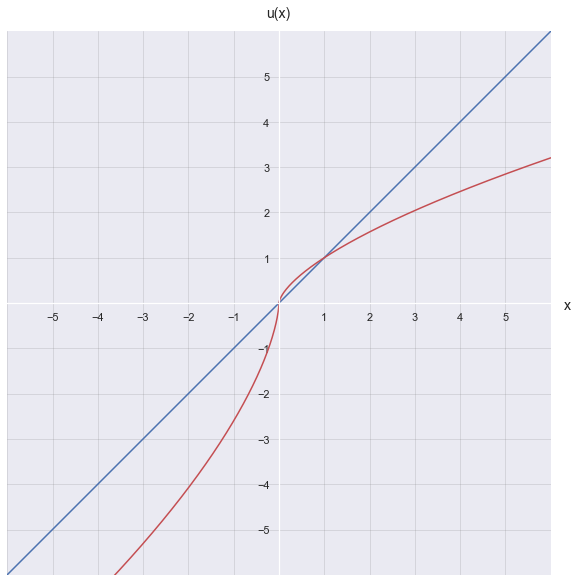}
  \caption{Conventional CPT utility functions; The plot shows $u^+(x)=x^{\alpha}$ for $x\geq 0$, and $-u^-(x)=-\lambda(-x)^{\beta})$ for $x<0$, with $\alpha=\beta=0.65$ and $\lambda=2.6$.}
 \label{fig:util}
\end{figure}

\subsection{Network Aggregative Markov Games}
Throughout this paper, we assume that agents are interacting in an ergodic network aggregative Markov game with a discounted infinite-horizon criterion. An NAMG with $N$ players is an MG denoted by $M=(S,N,A,R,P,G,\gamma, p_{s_0})$, where $S$ is the state space, $A=A_1\times ... \times A_N$ is the joint action space; $R: S\times A\times S \rightarrow \mathbb{R}^N$ is a joint reward function bounded in $[-R_{max}, R_{max}]$ where $R_{max}>0$; $P(.|s,a)$ is the MDP transition probability distribution; $\mathcal{G}(\mathcal{N},\mathcal{E})$ is a graph with edge set $\mathcal{E}$ on which each agent interacts with its neighbors; $\gamma$ is the MDP's discount factor; and $p_{s_0}$ is the initial state distribution. In NAMG, for each agent $n$, the reward function is a function of its own action and an aggregative function of other agents' actions,

\begin{equation}\label{rew}
     R^i(s,a^i,a^{-i}) = R^i(s,a^i,\sigma^i(a^{-i})),
\end{equation}

where we have,

\begin{equation}
\sigma^i(a^{-i}) = \sum_{j \in \mathcal{N} \backslash i}\omega_{ij} a^j,
\end{equation}

where $\omega$ are the edge weights of the communication graph $\mathcal{G}$, with $w_{ij}$ denoting the weight of the edge from $j$ to $i$. Therefore, given the graph, and by observing its neighbors' actions, agent $i$ is able to calculate $\sigma^i(a^{-i})$. Figure~\ref{fig:namg} shows a schematic of an NAMG.

\begin{figure}[h]
 \centering
 \includegraphics[width=10cm,height=5cm,keepaspectratio]{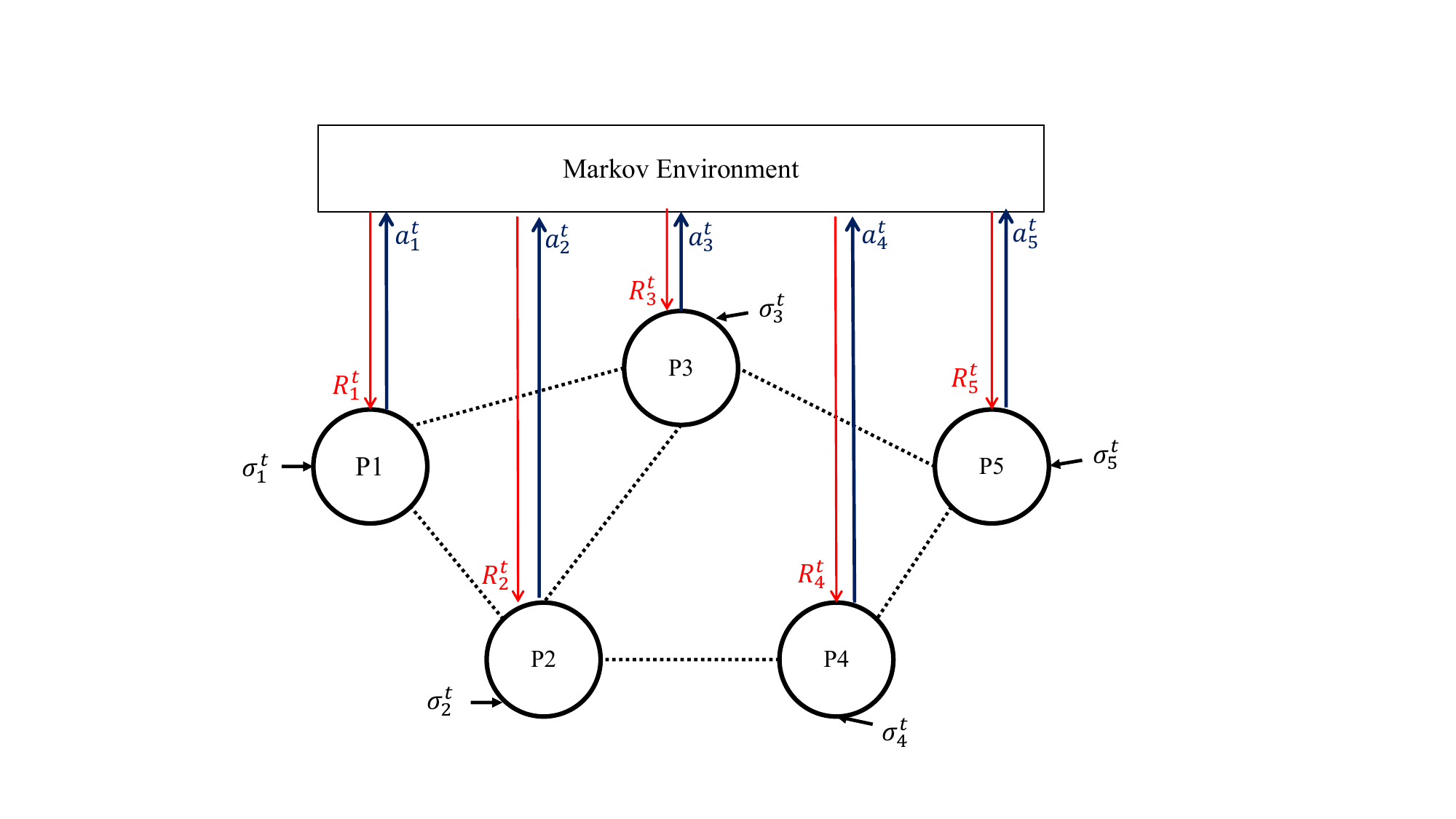}
  \caption{A network aggregative Markov game}
 \label{fig:namg}
\end{figure}

Previously, in various domains, such as resource allocation \cite{deng2019distributed}, social networks \cite{zhu2013pricing}, electrical microgrids \cite{tan2021distributed}, and power systems \cite{deng2021distributed}, either single-state network aggregative games (NAGs), or dynamic network aggregative games have been studied in risk-neutral setting. Furthermore, most of the theoretical works in this domain have focused on studying convergence to the Markov perfect Nash or Stackelberg equilibrium in single-state or dynamic NAGs with quadratic cost/reward functions that ensure the uniqueness of the equilibrium \cite{parise2020distributed, cenedese2020time, shokri2020leader, saffar2017pricing}. In this paper, for the first time, we consider risk-sensitive NAMGs.

\subsection{CPT Risk-Sensitive MARL Objective in NAMGs}

Using the nested formulation, the objective of the risk-sensitive agent $i$ in an NAMG $\max_{\pi^i} J^{\pi^i,\pi^{-i}}$ will be equivalent to 

\begin{equation} \label{eq:cpt-multi-obj-orig}
\begin{aligned}
	\max_{\pi^i} V^i_{\pi} (s_0)  &= \max_{\pi^i} \cpt_{\pi(s_0,.)\times \prob(.|s_0,a_0)} \big[R^i(s_{\tau},a_{\tau})+...\\
	&+\gamma^{\tau}\cpt_{\pi(s_{\tau},.)\times \prob(.|s_{\tau},a_{\tau})}\big[R^i(s_{\tau},a_{\tau})+...\big]\big]\\
	&= \max_{\pi^i} \cpt_{\pi(s_0,.)\times \prob(.|s_0,a_0)} \big[R^i(s_0,a_0)+ \gamma V^i_{\pi}(s_1)\big],
\end{aligned}
\end{equation}

where \( \pi(s_t,.) = \pi^i(s_t,.) \times \pi^{-i}(s_t,.) \) and \( a_t = (a^i_t, a^{-i}_t) \), respectively. Using the properties of NAMGs, and considering $\prob(\sigma^{-i}_0|s_0)$ as the probability that $\sigma^{-i}_0$ occurs at state $s_0$ for agent $i$, we can rewrite the objective as

\begin{equation}\small \label{eq:cpt-multi-obj}
\max_{\pi^i} V^i_{\pi} (s_0) = \max_{\pi^i} \cpt_{\pi^i(a^i_0|s_0)\times \prob(\sigma^{-i}_0|s_0) \times \prob(s_1|s_0,a_0)} \big[R^i(s_0,a_0)+ \gamma V^i_{\pi}(s_1)\big].
\end{equation}

Therefore, by observing the actions of neighboring agents and calculating the aggregative term $\sigma^{-i}$, agent $i$ can treat $\prob(\sigma^{-i}|s)$ as a probability distribution similar to the transition probabilities for each state.

\section{Distributed Nested CPT Policy Gradient} \label{sec:cptpg}
In this section, we derive a gradient expression for the Markov dynamic CPT risk measure in NAMGs, represented by the gradient of the initial state's value function in an ergodic CPT risk-sensitive NAMG, $\nabla V^i_{\pi_{\theta}}(s_0)$. Before presenting the PG theorem, we state the following assumption,

\begin{assumption}\label{assumption:w}
The weight functions $w^{\pm}$ are double differentiable, and the derivatives $w^{'}_{\pm}$ are Lipschitz continuous with common constant $L$. Furthermore, the utility functions $u^{\pm}$ are differentiable (denoted by $u^{'}_{\pm}$) for all agents.
\end{assumption}

The above assumption may seem strict at first. However, conventional forms of the CPT utility functions, specifically $u^+(x)=x^{\alpha}$ and $-u^-(x)=-\lambda(-x)^{\beta}$, along with the weighting functions $\omega^+(p) = \frac{p^{\gamma}}{(p^{\gamma} + (1-p)^{\gamma})^{(1/\gamma)}}$ and $\omega^-(p) = \frac{p^{\delta}}{(p^{\delta} + (1-p)^{\delta})^{(1/\delta)}}$, depicted in Figures \ref{fig:wieghting} and \ref{fig:util}, satisfy this assumption.

\begin{theorem} (Nested CPT Policy Gradient) \label{theo:cptpg}

Given Assumption~\ref{assumption:w}, the gradient of the CPT return for agent $i$, $V^i_{\pi_{\theta}}(s_0)$, with respect to the policy parameter $\theta^i$ is

\begin{equation} \label{eq:grad}
\begin{split}
    \nabla V^i_{\pi_{\theta}}(s_0) \propto& \expectation_{\mu_{cpt}^i(s)}\bigg[ \sum_{a,s'}\frac{\partial\phi}{\partial(\pi^i_{\theta}(a^i|s)\prob(\sigma^{-i}|s)\prob(s'|s,a))}\\&\prob(\sigma^{-i}|s)\prob(s'|s,a)(\nabla\pi_{\theta^i}(a^i|s)) u(R^i(s,a^i,\sigma^{-i},s')+\gamma V^i_{\pi_{\theta}}(s'))\bigg],
\end{split}
\end{equation}

where, $\phi$ and $u$ represent the CPT cumulative weighting and utility functions of the agent from \eqref{eq:cptdef} (superscript $i$ is dropped). The distribution $\mu_{cpt}^i$ is a \emph{subjective} steady-state probability distribution of the MDP in which $\mu_{cpt}^i(s) = \frac{\eta_{cpt}^i(s)}{\sum_{s \in S}\eta_{cpt}^i(s)}$, where $\eta_{\cpt}^i(s)$ is a subjective measure of time spent in each state and can be obtained by solving the following system of linear equations,

\begin{equation} \label{eq:eta}
     \eta_{cpt}^i(s) = p_0(s) + \sum_{\bar{s} \in S}\eta_{cpt}^i(\bar{s})\sum_{a \in A}\phi(\pi(a^i|\bar{s})\prob(\sigma^{-i}|\bar{s})\prob (s|\bar{s}, a))\frac{\partial u}{\partial V^i_{\pi_{\theta}}}(s),
\end{equation}

where $p_0(s)$ denotes the probability that the Markov chain starts in state $s$, and $u$ and $\phi$ are the utility cumulative weighting functions of agent $i$ from \eqref{eq:cptdef} ($u^{\pm}$ and $\phi^{\pm}$ are chosen according to the sign of $R^i(\bar{s},a,s)+\gamma V^i_{\pi_{\theta}}(s)$).

\end{theorem}

\begin{proof}
Considering agent $i$, we drop the subscript $i$ and denote $R^i(s_t,a_t,s_{t+1})$ as $R_t$ and $V^i_{\pi_{\theta}}(s_t)$ as $V_{\pi_{\theta}}(s_t)$. Furthermore, $\pi_{\theta}(s,a)$ used below, where $a$ is equivalent to $(a^i,a^{-i})$, is the more general case of joint policies in Markov games, which encompasses $\pi_{\theta}(a^i|s)\prob(\sigma^{-i}|s)$ in an NAMG. The gradient of the CPT risk-sensitive return considering its recursive definition can be written as
\begin{equation} \label{eq:gradexpand}
\small
\begin{split}
    \nabla V_{\pi_{\theta}}(s_0) &=\\ &\nabla \left[\sum_{a_0,s_{1}}\phi (\pi(s_0, a_0)\prob(s_{1}|s_0, a_0)) u(R_0 + V_{\pi_{\theta}}(s_{1}))\right]
     \\& = \sum_{a_0,s_{1}} \bigg[\nabla\phi (\pi(s_0, a_0)\prob(s_{1}|s_0, a_0)) u(R_0+V_{\pi_{\theta}}(s_{1})) \\&+ \phi (\prob(s_{1}|s_0, a_0)) \nabla u(R_0 + V_{\pi_{\theta}}(s_{1}))\bigg]
     \\& = \sum_{a_0,s_{1}} \Bigg[\nabla\phi (\pi(s_0, a_0)\prob(s_{1}|s_0, a_0)) u(R_0+V_{\pi_{\theta}}(s_{1})) \\&+ \phi (\pi(s_0, a_0) \prob(s_{1}|s_0, a_0))\frac{\partial u}{\partial V_{\pi_{\theta}}(s_1)}
     \\ & \nabla \bigg[\sum_{a_1,s_{2}} \phi (\pi(s_1, a_1)\prob(s_{2}|s_1, a_1)) u(R_1 + V_{\pi_{\theta}}(s_{2}))\bigg]\Bigg]
     \\ & = \sum_{a_0,s_1} \Bigg[\nabla\phi (\pi(s_0, a_0)\prob(s_{1}|s_0, a_0)) u(R_0 + V_{\pi_{\theta}}(s_{1})) \\&+ \phi (\pi(s_0, a_0)\prob(s_{1}|s_0, a_0))\frac{\partial u}{\partial V_{\pi_{\theta}}(s_1)}
     \\& \sum_{a_1,s_2}\bigg[\nabla\phi (\pi(s_1, a_1)\prob(s_{2}|s_1, a_1)) u(R_1+V_{\pi_{\theta}}(s_{2})) \\&+ \phi (\pi(s_1, a_1)\prob(s_{2}|s_1, a_1)) \nabla u(R_1+V_{\pi_{\theta}}(s_{2}))\bigg]\Bigg].
\end{split}
\end{equation}

We define 
\begin{equation}
\small
    DPr(s_0\rightarrow s_1, k=1, \pi_{\theta}):=\sum_{a_0}\phi(\pi(s_0, a_0)\prob(s_{1}|s_0, a_0))\frac{\partial u}{\partial V_{\pi_{\theta}}(s_1)}
\end{equation}

as the \emph{subjective (distorted)} visitation probability of $s_1$ right after $s_0$ following policy $\pi_{\theta}$. Note that since $u$ is a non-decreasing function with positive derivatives everywhere and $\phi$ is a function that maps $[0,1]$ to $[0,1]$, this term is always positive. By defining $DPr(s_0\rightarrow s_0, 0, \pi_{\theta}):=1$, by recursion, we can write the subjective probability of visiting state $s_{k+1}$ after $k+1$ steps, starting from $s_0$ and following policy $\pi_{\theta}$ as

\begin{equation}
\small
\begin{split}
    &DPr(s_0\rightarrow s_{k+1}, k+1, \pi_{\theta})=\\&\sum_{s_k} DPr(s_0\rightarrow s_{k}, k, \pi_{\theta})DPr(s_k\rightarrow s_{k+1},1, \pi_{\theta})
\end{split}
\end{equation}

Therefore, after repeated unrolling, we can write \eqref{eq:gradexpand} as
\begin{equation}
    \begin{split}
        &\nabla V_{\pi_{\theta}}(s_0) =
         \sum_{s}\Bigg(\bigg(\sum_{k=0}^{\infty}DPr(s_0\rightarrow s,k,\pi_{\theta})\bigg)\\&\sum_{a,s'}\nabla \phi(\pi(s,a)\prob(s'|s, a)) u(R(s,a,s') + V_{\pi_{\theta}}(s'))\Bigg).
    \end{split}
\end{equation}

Similar to a risk-neutral MDP, given $\pi_{\theta}$ and the state value function corresponding to this policy, the function $DPr$ is an inherent property of the CPT risk-sensitive MDP (this function can be compared with the function $Pr$ in the proof of risk-neutral policy gradient theorem in \cite{sutton2018reinforcement}, Section 13.2). Therefore, we let $\eta_{cpt}(s) := \sum_{k=0}^{\infty}DPr(s_0\rightarrow s,k,\pi_{\theta})$, which can be considered a \emph{subjective (perceived)} measure of time that the CPT risk-sensitive agent spends in state $s$ when following policy $\pi_{\theta}$ and starting from state $s_0$. In a similar fashion as in risk-neutral ergodic MDPs (see \cite{sutton2018reinforcement}, Section 9.2), $\eta(s)$ can be calculated by solving the following system of linear equations,

\begin{equation} \label{eq:eta}
     \eta_{cpt}(s) = p_0(s) + \sum_{\bar{s} \in S}\eta_{cpt}(\bar{s})\sum_{a \in A}\phi(\pi(\bar{s},a)\prob (s|\bar{s}, a))\frac{\partial u}{\partial V_{\pi_{\theta}}}(s),
\end{equation}

where $p_0$ is the probability distribution of the starting state. Therefore, we can write \eqref{eq:gradexpand} as
\begin{equation}
    \begin{split}
        \nabla V_{\pi_{\theta}}(s_0) =
         \sum_{s}\eta_{cpt}(s)\sum_{a,s'}&\nabla \phi(\pi(s,a)\prob(s'|s, a))\\& u(R(s,a,s') + V_{\pi_{\theta}}(s')).
    \end{split}
\end{equation}

As $\eta_{cpt}(s)$ is positive for all $s$, we can define $\mu_{cpt}(s) = \frac{\eta_{cpt}(s)}{\sum_{s \in S}\eta_{cpt}(s)}$ as the \emph{subjective} limiting (steady-state) distribution of the CPT risk-sensitive MDP, and therefore, we have

\begin{equation} \label{eq:gradcompare}
    \begin{split}
        \nabla V_{\pi_{\theta}}(s_0) \propto\expectation_{\mu_{cpt}(s)}\bigg[\sum_{a,s'}&\nabla \phi(\pi(s,a)\prob(s'|s, a))\\& u(R(s,a,s') + V_{\pi_{\theta}}(s'))\bigg].
    \end{split}
\end{equation}

It is interesting to compare \eqref{eq:gradcompare} with the similar expression in risk-neutral policy gradient theorem, $\expectation_{\mu(s)}\bigg[\sum_{a}\nabla \pi(s,a)Q(s,a)\bigg]$. Due to the non-linear CPT operator (compared to the linear expectation operator), the policy is entangled with the transition probabilities inside the gradient of the cumulative weighting function, and therefore, in the risk-sensitive case, it is not possible to define a stand-alone Q-function as a function of state and action to measure the quality of an action in a given state. As noted by Lin \cite{lin2013stochastic}, this complication has the consequence that the optimal risk-sensitive policy even in the single-agent setting can be stochastic. To further expand the above expression, we can use the chain rule of calculus and write,

\begin{equation} \label{eq:gradproof}
\begin{split}
    \nabla V_{\pi_{\theta}}(s_0) \propto& \expectation_{\mu_{cpt}(s)}\bigg[ \sum_{a,s'}\frac{\partial\phi}{\partial(\pi_{\theta}(s,a)\prob(s'|s,a))}\\&\prob(s'|s,a)(\nabla\pi_{\theta}(s,a)) u(R(s,a,s')+\gamma V_{\pi_{\theta}}(s'))\bigg].
\end{split}
\end{equation}

This is the general case of PG in CPT risk-sensitive MARL. Given the aggregative term $\sigma^{-i}$ in NAMGs, we can rewrite this equation as~\eqref{eq:grad}.

\end{proof}

We now provide an algorithm to estimate the above gradient using samples from a simulator of the environment or a large enough experience dictionary. This approximation scheme which is later used to also estimate the value function is Algorithm~\ref{alg:CPTVEst} is proposed by Jie et al. \cite{jie2018stochastic} to estimate the CPT value of an r.v., $X$, using samples from its distribution. The following Assumption (A2 in \cite{jie2018stochastic}) is needed to guarantee the asymptotic consistency of this estimation algorithm.

\begin{assumption}\label{assumption:u}
The utility functions $u^{+}$ and $u^{-}$ are continuous and non-decreasing on their support $\mathbb{R^+}$ and $\mathbb{R^-}$, respectively.
\end{assumption}

The above assumption also holds for conventional forms of weighting and utility functions in Figures \ref{fig:wieghting} and \ref{fig:util}  \cite{jie2018stochastic}. Given Assumptions~\ref{assumption:w} and \ref{assumption:u}, Proposition 4 in \cite{jie2018stochastic} is verified and for a given r.v. $X$, Algorithm~\ref{alg:CPTVEst} is guaranteed to have asymptotic consistency, i.e., it converges to $\cpt[X]$ asymptotically as the number of samples, $n$, approaches infinity.

\begin{algorithm} [tb]
\caption{CPT Value Estimation}
\label{alg:CPTVEst}
\begin{algorithmic}[1]
\STATE\textbf{Require:} Samples $X_1$,..., $X_n$ from the distribution of r.v. $X$, sorted in ascending order.
\STATE  Let
\begin{equation*}
\begin{split}
    \hat{\rho}_{cpt}^+ &:= \sum^n_{i=1} u^+(X_i)\bigg(\omega^+\bigg(\frac{n+1-i}{n}\bigg)-\omega^+\bigg(\frac{n-i}{n}\bigg)\bigg),\\
    \hat{\rho}_{cpt}^- &:= \sum^n_{i=1} u^-(X_i)\bigg(\omega^-\bigg(\frac{i}{n}\bigg)-\omega^-\bigg(\frac{i-1}{n}\bigg)\bigg).
\end{split}
\end{equation*}
\STATE Return $\hat{\rho}_{cpt} = \hat{\rho}_{cpt}^+ - \hat{\rho}_{cpt}^-$.
\end{algorithmic}
\end{algorithm}

\paragraph{\textbf{Gradient estimation.}} To have a estimate of the gradient in~\eqref{eq:grad}, we need estimates of CPT values corresponding to \\$\phi(\pi(a^i|\bar{s})\prob(\sigma^{-i}|\bar{s})\prob (s|\bar{s}, a))$ and $\frac{\partial\phi}{\partial(\pi_{\theta}(a^i|s)\prob(\sigma^{-i}|\bar{s})\prob(s'|s,a))}$ (which we denote by $\phi'\left(\pi_{\theta}(a^i|s)\prob(\sigma^{-i}|\bar{s})\prob(s'|s,a)\right)$). Note that Assumption \ref{assumption:w} states that $\omega'_{\pm}$ are Lipschitz and therefore, they can be used as independent CPT weighting functions with corresponding cumulative weighting functions $\phi'_{\pm}$. Given the transition probabilities and repeated distributed sampling of rewards and transitions by agents from the environment or the experience dictionaries, the term in brackets corresponding to each state can be estimated using Algorithm~\eqref{alg:CPTVEst}. Furthermore, using these samples and solving a linear system of equations resulting from \eqref{eq:eta}, the subjective steady-state distribution $\mu_{cpt}(s)$ can be found. We note that this estimation algorithm is model-based and requires transition probabilities, however, it does not assume any knowledge of the reward function or the policies of other agents, and is therefore privacy-preserving. Having a policy gradient theorem and a corresponding gradient approximation scheme, we can now develop our distributed actor-critic algorithm.

\section{Distributed Nested CPT Actor-Critic}

Algorithm~\eqref{alg:main} lays out the pseudocode for Distributed Nested CPT Actor-Critic. As can be seen, the critic's value function is estimated using the sampling strategy in Algorithm~\ref{alg:CPTVEst}, and we use the samples from the simulator for bootstrapping (by adding them to the experience dictionary for later use). Although sampling from the simulator for gradient and value function approximation can be computationally intensive, it can become less so as we build the experience dictionary and do away with the simulator. We now prove the asymptotic convergence of the proposed algorithm.

\begin{algorithm*} [h]
\caption{Distributed Nested CPT Actor-Critic}
\label{alg:main}
\begin{algorithmic}[1]
\STATE\textbf{Inputs:} \emph{shared among agents:} Initial state $s^0$, number of samples used for CPT estimation ($n_{max}$), learning rate sequences ($\{\alpha_{cr,t}\}_{t\geq 0},\{\alpha_{ac,t}\}_{t\geq 0}$), and transition probabilities. \emph{Local variables for agent $n$:} Initial $V_{\pi_{\theta_0}}^n$, initial policy parameters ($\theta^n_0$), and an empty experience dictionary $ExpDict^n$.
\STATE \textbf{For each agent $n$, do:}
\STATE Sample action $a^n_0$ from policy $\pi_{\theta^n_0}(.|s_0)$.
\STATE $t \leftarrow 0$.
\STATE\textbf{Repeat}
\STATE \quad Sample $a^n_{t}$ from $\pi_{\theta^n_{t}}(.|s_{t})$.
\STATE \quad Execute $a^n_t$ and observe $r^n_t$, $s_{t+1}$, and $\sigma^{-n}_t$.
\STATE \quad Push $(r_t, s_{t+1},\sigma^{-n}_t)$ to $ExpDict^n(s_t,a^n_t,\sigma^{-n}_t)$.
\STATE \quad \textbf{Critic value estimation:}
\STATE \quad Create empty array $X$ of size $n_{max}$.
\STATE \quad \textbf{for} each $i=1,2,...,n_{max}$, \textbf{do}
\STATE \quad \quad Sample $\hat{a}^n_{t}$ from $\pi_{\theta^n_{t}}(.|s_{t})$ and construct $\hat{\sigma}^{-n}_t$ by observing neighbors.
\STATE \quad \quad Sample $(\hat{r}^n_t, \hat{s}_{t+1})$ from $ExpDict(s_t,\hat{a}^n_t,\hat{\sigma}^{-n}_t)$ if it is large enough, and otherwise from a simulator of the environment.
\STATE \quad \quad Let $X_i = \hat{r}^n_t + \gamma V_{\pi_{\theta}}^n(\hat{s}_{t+1})$.
\STATE \quad \quad If the sample came from a simulator, push $(\hat{r}^n_t, \hat{s}_{t+1})$ to $ExpDict(s_t,\hat{a}^n_t,\hat{\sigma}^{-n}_t)$ for later use.
\STATE \quad \textbf{end for}
\STATE \quad Estimate $\hat{V}_{\pi_{\theta_t}}^n(s_t)$ using Algorithm~\ref{alg:CPTVEst}.

\STATE \quad \textbf{Critic step:}
\STATE \quad Calculate the TD-error:
\begin{equation*}
    \delta_t := \hat{V}_{\pi_{\theta_t}}^n(s_t) - V_{\pi_{\theta_t}}^n(s_t).
\end{equation*}
\begin{equation*}
    V_{\pi_{\theta_t}}^n(s_t) \leftarrow V_{\pi_{\theta_t}}^n(s_t) + \alpha_{cr,t}\delta_t.
\end{equation*}

\STATE \quad \textbf{Actor step:}
\STATE \quad Compute $\nabla V^n_{\pi_{\theta_t}}(s_0)$ using the gradient estimation scheme described in Section 4.
\begin{equation*}
    \theta_{t+1}^n := \theta_t^n + \alpha_{ac,t}\nabla V^n_{\pi_{\theta_t}}(s_0).
\end{equation*}
\STATE \quad $t \leftarrow t+1.$
\STATE \textbf{Until convergence}
\end{algorithmic}
\end{algorithm*}

\subsection{Convergence of the Critic} \label{sec:critic}
In order to calculate the state value function corresponding to the current policy, we define the following $TD(0)$ CPT operator (note that the agent's superscript $n$ has been dropped),

\begin{equation} \label{eq:oper}
    T_{cpt}V_{\pi_{\theta}}(s) = \cpt_{\pi_{\theta}(.|s)\times \prob (.|s,a)}\left[R(s,a,s') + \gamma V_{\pi_{\theta}}(s')\right] 
\end{equation}

The following assumption is needed to ensure that the operator in~\eqref{eq:oper} is a sup-norm contraction.

\begin{assumption}\label{assumption:crit}
The utility functions $u^+$ and $u^-$ are invertible (denoted by $u_{+}^{-1}$ and $u^{-1}_{-}$) and differentiable (denoted by $u^{'}_{+}$ and $u^{'}_{-}$), and we have $u^+(0)=u^-(0)=0$. Further, there exists $\beta \in (0,1)$ such that $\int_0^{\gamma c} \omega^+(\prob(X<x))u^{'}_{+}(\gamma c -x)dx + \int_0^{\gamma c} \omega^-(\prob(X>x))u^{'}_{-}(x)dx \leq \beta c$ holds for any $c>0$ and any non-negative real-valued r.v. $X$, where $\gamma$ is the discount factor of the MDP.
\end{assumption}

Similar to Assumptions \ref{assumption:w} and \ref{assumption:u}, the above assumption can also be verified to hold for typical analytical forms of $\omega^{\pm}$ and $u^{\pm}$ in Figures \ref{fig:util} and \ref{fig:wieghting} as shown by Lin et al. \cite{lin2018probabilistically}. Under this assumption, based on Theorem 6 in \cite{lin2018probabilistically}, the $TD$ operator \eqref{eq:oper} is a sup-norm contraction on a Banach space defined over the MDP's state space that includes all possible state value functions $V_{\pi_{\theta}}$. Therefore, for every $V_{\pi_{\theta}}, \bar{V}_{\pi_{\theta}} \in B(S)$, there exists $\alpha \in (0,1)$ such that

\begin{equation}
    \|T_{cpt}V_{\pi_{\theta}}-T_{cpt}\bar{V}_{\pi_{\theta}}\|_{\infty}\leq \alpha \|V_{\pi_{\theta}}-\bar{V}_{\pi_{\theta}}\|_{\infty}
\end{equation}

\begin{remark}{(\emph{Applicability of linear function approximation})}
The contraction of the $TD$ operator \eqref{eq:grad} has only been validated for a tabular representation of the state-value function \cite{lin2018probabilistically}. We also assessed the possibility of approximating the state value function using linear functions for scaling up the proposed actor-critic algorithm to large or continuous state spaces. The traditional proof of convergence for a $TD$ critic with linear function approximation requires the contraction of this operator with respect to the $L^2$ norm defined over the steady-state distribution of the MDP (Lemma $4$ in Tsitsiklis and Van Roy \cite{tsitsiklis1997analysis}). However, via counterexample, it can be seen that this property does not necessarily hold for the $TD$ operator \eqref{eq:grad}.
\end{remark}

\begin{remark} The previous remark and the fact that we were required to limit ourselves to tabular representations implies that mathematical properties of CPT-sensitive MDPs do not allow them to belong to the family of robust MDPs \cite{osogami2012robustness} and enjoy properties such as linear approximation for the state value function and a convenient gradient estimation scheme as with coherent risk measures \cite{tamar2015policy}. It would be interesting to study and look at this limitation from a cognitive perspective and to see whether dealing with dynamic CPT risk-sensitive continuous control is cognitively cumbersome for humans in behavioral experiments.
\end{remark}

\subsection{Convergence of the Actor}

For notational simplicity, we denote $V^i_{\pi_{\theta}}(s)$ by $V_i(\theta,s)$. We prove the convergence of our AC algorithm to a subjective MPNE of the game if the following assumptions are satisfied.

\begin{assumption}{}\label{assumption:multconv}
For each agent $i$, $V_i(\theta,s)$ is convex with respect to $\theta^i$. Also, the gradient is uniformly bounded, i.e., for each agent $i$ there exists $\xi_i$ such that,
\begin{equation}
	\sup_{s \in\mathcal{S}}\left\|\nabla_{\theta_i} V_i(\theta,s)\right\| \leq \xi_i.
\end{equation}
\end{assumption}

\begin{assumption}{}\label{assumption:multilip}

The pseudo-gradient mapping value function,
$\nabla_\theta V(\theta, s)=\operatorname{col}\left(\nabla_{\theta_i} V_1(\theta, s), \ldots, \nabla_{\theta_N} V_N(\theta, s)\right)$,
is strongly monotone with respect to $\theta$, i.e., for every
$\theta, \theta' \in \Theta, s \in \mathcal{S}$, there exists $\mu > 0$ such that
\begin{equation*}
	\left(\nabla_\theta V(\theta, s)-\nabla_\theta V\left(\theta^{ \prime}, s\right)\right)^{\top}\left(\theta-\theta^{\prime}\right) \geq \mu\left\|\theta-\theta^{\prime} \right\|^2.
\end{equation*}

Furthermore, this mapping is Lipschitz continuous, i.e.,
\begin{equation*}
\left\|\nabla_\theta V(\theta, s)-\nabla_\theta V\left(\theta^{\prime}, s\right)\right\| \leq L_a\left\|\theta-\theta^{\prime}\right\|.
\end{equation*}
\end{assumption}

\begin{theorem}{}\label{theo:main}(\textbf{Convergence of the actor})
Given Assumptions \ref{assumption:multconv} and \ref{assumption:multilip} and a critic and an actor with learning steps such that,
\begin{equation}
\sum_{t=0}^{\infty}\alpha_{ac,t} = \infty, \quad \sum_{t=0}^{\infty}\alpha_{cr,t} = \infty ,\\
\quad \sum_{t=0}^{\infty}\alpha_{cr,t}^2 < \infty , \quad \sum_{t=0}^{\infty}\alpha_{ac,t}^2 < \infty , \quad \lim_{t\rightarrow \infty}\frac{\alpha_{ac,t}}{\alpha_{cr,t}} = 0,
\end{equation}

algorithm \eqref{alg:main} converges to the unique subjective Markov perfect Nash equilibrium of the NAMG, asymptotically.

\end{theorem}

\proof
We prove that the actor update will converge to the unique Nash policy of the Markov game, which exists if Assumptions \ref{assumption:multconv} and \ref{assumption:multilip} are satisfied, starting from any initial condition. We rewrite the actor update for agent $n$ below,

\begin{equation*}
\theta_{t+1}^n := \theta_t^n + \alpha_{ac,t}\nabla_{\theta}V^n(\theta,s_0).
\end{equation*}

Under Assumptions 4 and 5, based on Theorem 2 in Rosen \cite{rosen1965existence}, there exists a unique MPNE for the NAMG. Note that this CPT-sensitive (subjective) MPNE can be different from the MPNE of the game when the agents are risk-neutral. Consider the parameter vector $\theta^*$ as the vector that constructs this unique Nash policy, for which we have $\nabla_{\theta_i}J_i(\theta^*) = 0$ for all $i$. Per Assumption \ref{assumption:multconv} and by defining
$\Delta \theta_n^t = \theta^t_n - \theta^*_n$, we have

\begin{equation}\label{eq:tmp}
\begin{aligned}
\left\|\Delta \theta_n^{t+1}\right\|^2&=\left\|\Delta \theta_n^t-\alpha_{ac,t} \nabla_{\theta_n} V_n\left(\theta^t\right)\right\|^2 \\
& =\left\|\Delta \theta_n^t\right\|^2+\left(\alpha_{ac,t}\right)^2\left\|\nabla_{\theta_n} V_n\left(\theta ^t\right)\right\|^2 \\
& -2\alpha_{ac,t} \nabla_{\theta_n}\left(V_n\left(\theta^t\right)\right)^{\top} \Delta \theta_n^t \\
& \leq\left\|\Delta \theta_n^t\right\|^2+\left(\alpha_{ac,t}\right)^2 \xi_n^2-2 \alpha_{ac,t} \nabla_ {\theta_n} V_n\left(\theta^t\right)^{\top} \Delta \theta_n^t.
\end{aligned}
\end{equation}

By summing the above equation over different $n \in \{1,...,N\}$ and defining $\Delta \theta^t=\operatorname{col}\left(\Delta \theta_1^t \ldots, \Delta \theta_N^t\right) $, we have

\begin{equation}
\left\|\Delta \theta^{t+1}\right\|^2 \leq\left\|\Delta \theta^t\right\|^2+\sum_{n \in \mathcal{N} } \xi_n^2\left(\alpha_{ac,t}\right)^2-2 \alpha_{ac,t}\nabla_\theta V^{t^{\top}} \Delta \theta^t,
\end{equation}
   
where $ \nabla_\theta V^t= \operatorname{col}\left(\nabla_{\theta_1} V_1\left(\theta^t\right), \ldots, \nabla_{\theta_N} V_N\left(\theta ^t\right)\right) $.
We also know that
$\nabla_\theta V^*=\operatorname{col}\left(\nabla_{\theta_1} V_1\left(\theta^*\right), \ldots, \nabla_{\theta_N} V_N\left(\theta ^*\right)\right)=0$. Therefore, according to assumption \ref{assumption:multilip},

\begin{equation}
\begin{aligned}
-\nabla_\theta V^{t^{\top}} \Delta \theta^t & =-\left(\nabla_\theta J^t-\nabla_\theta V^*\right)^{\top} \Delta \theta^t \\
& =-\left(\nabla_\theta V^t\left(\theta^t\right)-\nabla_\theta V^t\left(\theta^*\right)\right)^{\top} \Delta \theta^t \\
& \leq-\mu\left\|\Delta \theta^t\right\|^2 .
\end{aligned}
\end{equation}

Finally, with telescopic summation,

\begin{equation}
\begin{aligned}
\lim _{t \rightarrow \infty}\left\|\Delta \theta^t\right\|^2 \leq\left\|\Delta \theta^0\right\|^2 & +\underbrace{\sum_{n \in \mathcal{N}} \xi_n^2 \sum_{t=0}^{\infty}\left(\alpha_{ac,t}\right)^2}_{T_1} \\
& -\underbrace{\sum_{t=0}^{\infty} \mu \alpha_{ac,t}\left\|\Delta \theta^t\right\|^2}_{T_2} .
\end{aligned}
\end{equation}

Since $\mu > 0$, $T_2 \geq 0$. Therefore, as $T_1$ is bounded, $T_2$ is bounded as well. Consequently, as $\sum_{t=0}^{\infty} \alpha_{ac,t} = \infty$,
$\lVert \Delta \theta^t \rVert^2$ converges to $0$ as $t \rightarrow \infty$ and as a result $\theta$ converges to $\theta^*$.
\endproof

Given the asymptotic proof of convergence for the actor and the critic and considering the conditions of the learning step sequences in Theorem~
\ref{theo:main}, we can apply Theorem 1.1 of Borkar \cite{borkar1997stochastic}, which shows the asymptotic convergence of the AC algorithm to the unique MPNE of the NAMG. Note that if Assumptions \ref{assumption:multconv} and \ref{assumption:multilip} do not hold, given the actor's gradient expression, we can only ensure the convergence of the AC algorithm to a locally optimal policy parameter for each agent.

\section{Numerical Experiment}
To examine the empirical convergence of Distributed Nested CPT AC, we constructed a risk-sensitive NAMG with an interpretable design for measuring the effect of loss aversion in risk-sensitive agents on the converged policies. Note that due to the CPT operator, Assumptions~\ref{assumption:multconv} and \ref{assumption:multilip} are hard to verify in any experimental setup and we did not expect the algorithm to converge to the subjective MPNE (which may not be unique if the aforementioned assumptions are not satisfied). In the constructed NAMG, there are four agents, five states, and three possible actions $(\mathcal{N} = 4, \mathcal{S}=\{0,1,2,3,4\}, \mathcal{A}=\{0,1,2\})$, and the communication graph is fully-connected. The reward function is defined as

\begin{equation}
R^i(s,a^i, \sigma^{i}(a^{-i})) = R_{self}^i(s) + \sigma^i(a^{-i}) R^i_{com}(s)a^i, 
\end{equation}

where the first term is the reward solely affected by the agent's action, and the second term is the reward that is affected by the actions of the neighboring $community$ of the agent. The aggregative term is

\begin{equation}
	\sigma^i(a^{-i}) = \frac{1}{N-1}( \sum_{j \in \mathcal{N} \backslash i}a_j),
\end{equation}

which indicates a communication graph with equal weights among the neighbors. The reward coefficient $R_{self}^i(s)$ for agent \(i\) is randomly generated from

\begin{equation}
	R_{self}(s,a^i) \sim Normal(0.5,0.1), \forall i \in {1,...,N}.
\end{equation}

Also, the reward coefficients $R_{com}^i(s)$ is randomly generated form

\begin{equation}
	R_{com}^i(s) \sim  5 \cdot uniform[-0.5,0.5].
\end{equation}

The above setup implies a high risk for the agent if it decides to take an action greater than \(a^i = 0\) and become involved with its neighboring community, e.g., take a financial risk in an interactive market. Thus, it can be said that each agent in this non-cooperative environment chooses by its action how much risk it wants to take and to what degree it wants to interact with the community (other agents who could, for instance, be economic, political, or social competitors), and ties its received reward to their actions. If the agent chooses the action \(a^i = 0\), it will settle for a low, but risk-free profit. However, when choosing another action, depending on the actions of the other agents, it can make a significant profit or loss that is also affected by stochasticity of the environment. Our objective is to study the agents' risk-aversion levels based on the parameter \(\lambda\) in the cumulative perspective theory utility function (Figure~\ref{fig:util}). The higher the value of this parameter, the more loss-averse the agent is. We expect that in the designed non-cooperative environment, a more loss-averse agent will choose the action \(a = 0\) with a higher probability and have a tendency to become socially isolated or conservative. To evaluate this hypothesis, we run the proposed algorithm for this risk-sensitive NAMG and for four different loss-aversion scenarios. In the first scenario, all agents are risk-neutral (corresponding to a vanilla AC algorithm with linear function approximation). In the second scenario, all agents have the same level of loss aversion ($\lambda=2.6$). In the third scenario, only Agent 1 is risk-sensitive ($\lambda=2.6$), and other agents are risk-neutral. Finally, in the last scenario, all agents are risk-sensitive, but Agent one has a higher loss aversion coefficient ($\lambda=3.2$), while others have $\lambda=2.6$. Figure~\ref{fig:vals} shows the convergence of the value functions corresponding to one of the states in the second scenario. Figure~\ref{fig:pols} shows the probability of choosing action $0$ (a quantitative indicator of social conservatism) in the converged policies of agents in each scenario. As observed, the level of social isolation and the probability of choosing the conservative action \(a = 0\) is proportional to the risk-aversion level of the agents in the community.

\begin{figure}[htbp]
 \centering
 \includegraphics[scale=0.2]{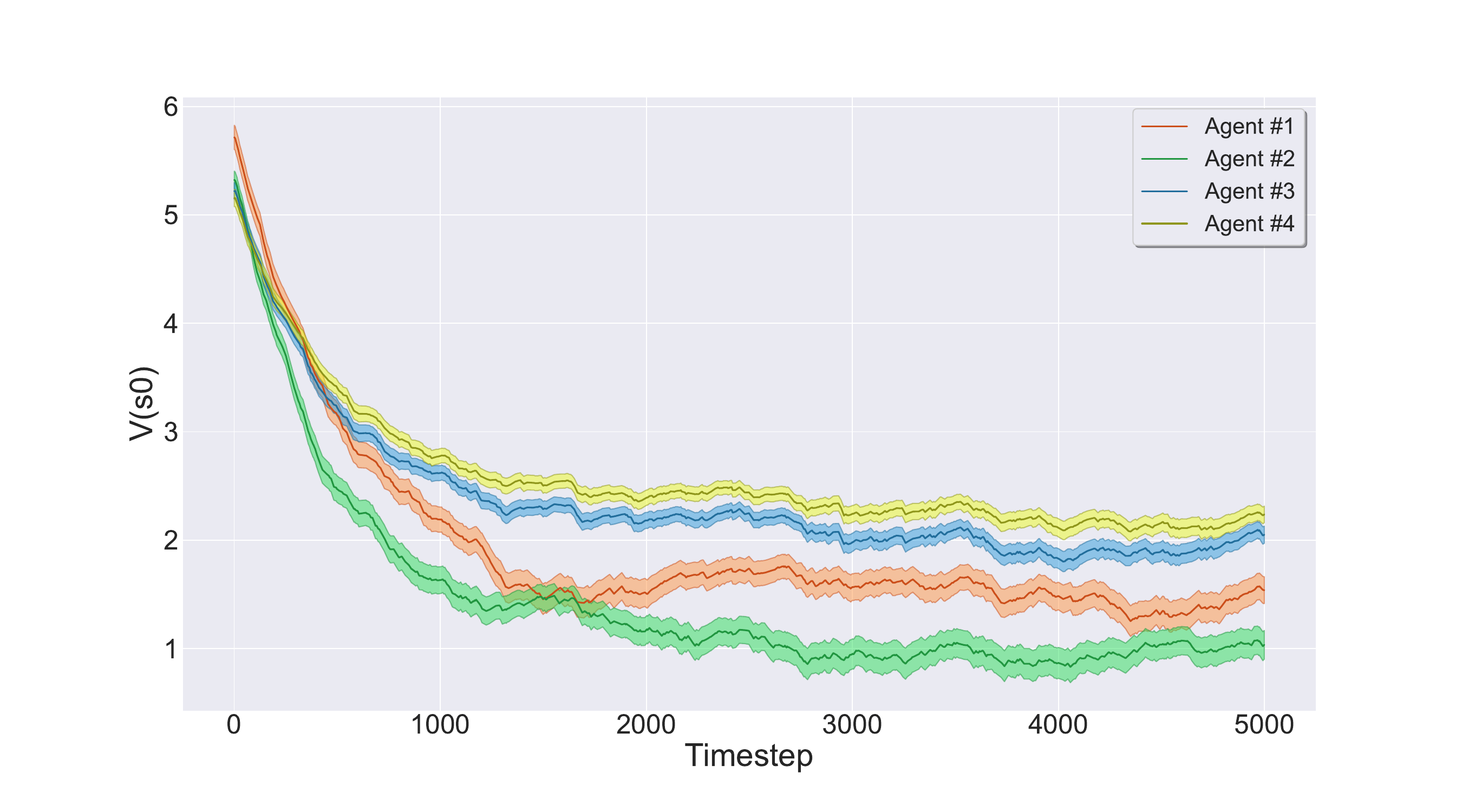}
  \caption{Smoothed mean value function of a given state over eight independent runs in Distributed Nested CPT-AC for scenario 2 (all agents are risk-sensitive with $\lambda=2.6$).}
 \label{fig:vals}
\end{figure}

\begin{figure}[htbp]
 \centering
 \includegraphics[scale=0.2]{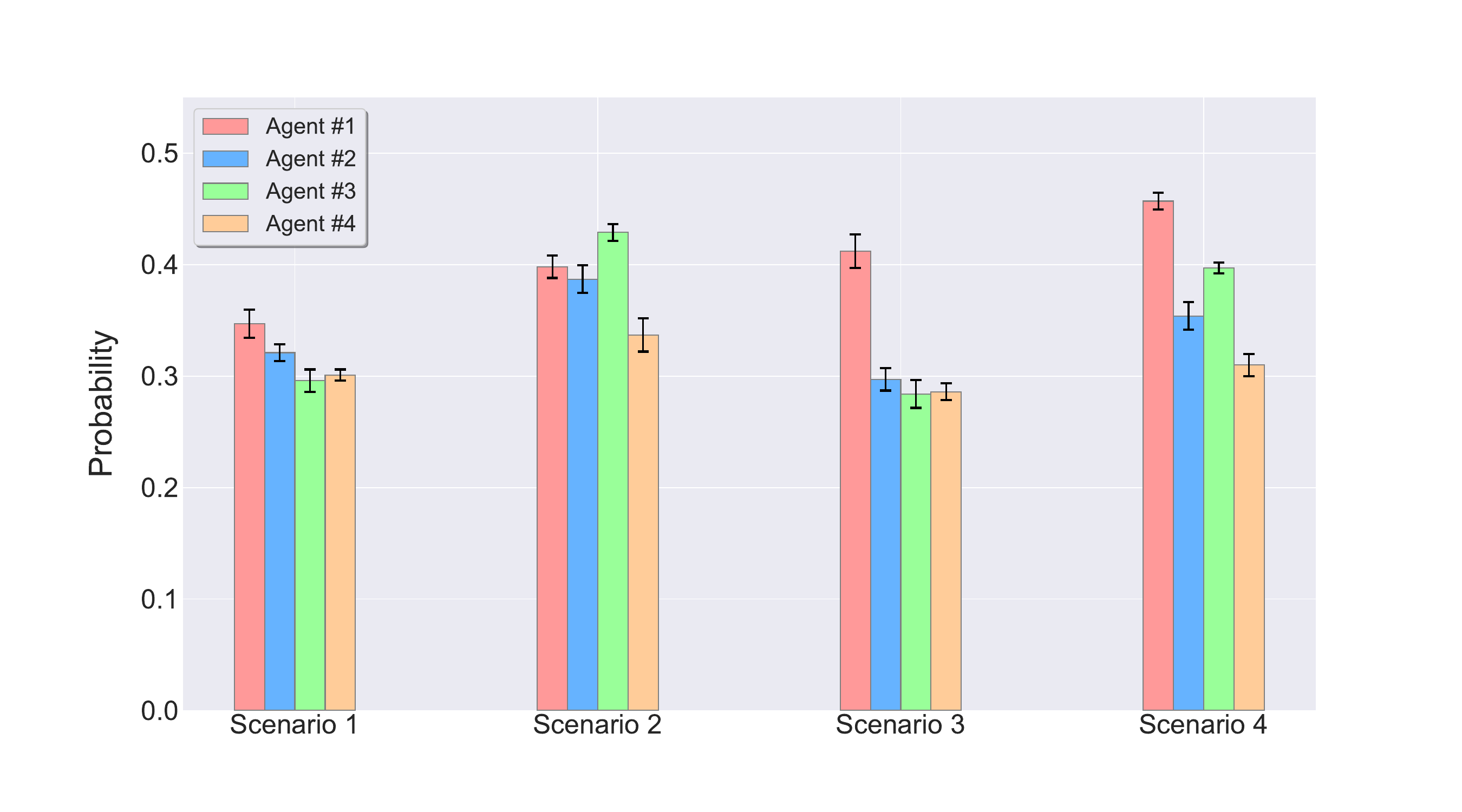}
  \caption{Mean converged policies over eight independent runs for different loss aversion scenarios. Scenario 1: all agents are risk-neutral, scenario 2: all agents are risk-sensitive ($\lambda=2.6$), scenario 3: only Agent 1 is risk-sensitive ($\lambda=2.6$), scenario 4: Agent 1 has a higher loss aversion coefficient ($\lambda=3.2$) than others ($\lambda=2.6$).}
 \label{fig:pols}
\end{figure}

\section{Conclusion and Future Works}
In this work, we proposed a distributed risk-sensitive MARL algorithm in NAMGs with theoretical convergence guarantees based on cumulative prospect theory, a cognitive risk measure that broadens the scope of the traditionally adopted coherent risk measure. We empirically showed the positive correlation between loss aversion and social isolation of agents. We observed that scaling the proposed algorithm to larger environments and continuous control is not compatible with theoretical convergence gurantees. However, a plausible direction of future work is the appropriate use of function approximation and deep RL methods to tackle the curse of dimensionality for large state and action spaces in CPT-sensitive RL, and in general risk-averse RL \cite{urpi2021risk}, in an empirical framework, albeit without theoretical convergence gurantees.









\bibliographystyle{ACM-Reference-Format} 
\bibliography{aamas}


\end{document}